\documentclass[letterpaper, 10 pt, conference, onecolumn]{ieeeconf}  

\IEEEoverridecommandlockouts                              

\usepackage{import}
\usepackage{algorithm}
\usepackage{algpseudocode}
\usepackage{amsmath, amssymb}

\usepackage{amsthm}

\usepackage{hyperref}      
\usepackage[capitalise]{cleveref} 
\usepackage{graphicx}
\usepackage{xcolor}
\usepackage{subcaption}
\usepackage{cite}

\theoremstyle{plain}
\newtheorem{theorem}{Theorem}
\newtheorem{proposition}[theorem]{Proposition}
\newtheorem{lemma}[theorem]{Lemma}
\newtheorem{corollary}[theorem]{Corollary}
\theoremstyle{definition}
\newtheorem{definition}[theorem]{Definition}

\theoremstyle{remark}

\title{\LARGE \bf
Contraction-Aligned Analysis of Soft Bellman Residual Minimization
with Weighted $L_p$-Norm for Markov Decision
Problem
}

\author{Hyukjun Yang, Han-Dong Lim and Donghwan Lee
\thanks{H. Yang, H. Lim and D. Lee are with the Department of Electrical Engineering, Korea Advanced Institute of Science and Technology (KAIST), Daejeon, 34141, South Korea
        {\tt\small \{jundol32, limaries30, donghwan\}@kaist.ac.kr}}%
}

\begin{document}

\maketitle
\thispagestyle{empty}
\pagestyle{empty}

\begin{abstract}
The problem of solving Markov decision processes under function approximation remains a fundamental challenge, even under linear function approximation settings. A key difficulty arises from a geometric mismatch: while the Bellman optimality operator is contractive in the $L_\infty$-norm, commonly used objectives such as projected value iteration and Bellman residual minimization rely on $L_2$-based formulations. To enable gradient-based optimization, we consider a soft formulation of Bellman residual minimization and extend it to a generalized weighted $L_p$-norm. We show that this formulation aligns the optimization objective with the contraction geometry of the Bellman operator as $p$ increases, and derive corresponding performance error bounds. Our analysis provides a principled connection between residual minimization and Bellman contraction, leading to improved control of error propagation while remaining compatible with gradient-based optimization.
\end{abstract}

\section{INTRODUCTION}

Solving Markov decision processes (MDPs) and their sample-based counterpart, reinforcement learning (RL), remains a central challenge, particularly when the state or action space is large. In such settings, function approximation is often necessary, but introduces theoretical and algorithmic challenges even in the linear setting~\cite{sutton1998reinforcement}. To address the challenges introduced by function approximation, a broad range of methods has been developed. Among them, two prominent classes are dynamic programming (DP)-based methods~\cite{bertsekas1995neuro, bertsekas2012dynamic_1}, which typically involve a projection step onto a function class when combined with function approximation, and Bellman residual minimization (BRM), which directly minimizes the Bellman residual~\cite{baird1995residual, maillard2010finite}. However, both approaches face fundamental difficulties under function approximation. In DP-based methods, although the Bellman operator is a contraction under the $L_\infty$-norm~\cite{schweitzer1985generalized, sutton1998reinforcement, puterman2014markov, bertsekas2012dynamic_1}, the projection step is typically defined with respect to the $L_2$-norm, and this mismatch may destroy the contraction property even under linear function approximation~\cite{gordon1995stable, de2000existence, scherrer2010should}. In BRM-based methods, the projection step is avoided, but the objective is still typically defined through the squared Bellman residual, which again corresponds to an $L_2$-based geometry~\cite{baird1995residual, scherrer2010should, antos2008learning, maillard2010finite, geist2017bellman, fujimoto2022should, lee2026analysis}. As a result, neither class is fully aligned with the contraction structure of the Bellman operator governing error propagation.

Building on this observation, we consider replacing the conventional $L_2$-based formulation with a more general weighted $L_p$-norm. We show that such a replacement is generally incompatible with DP-based methods, since the projection defined in weighted $L_p$-norm spaces generally break the contraction structure. In contrast, this formulation is naturally suited to BRM-based methods. More importantly, we demonstrate that adopting a weighted $L_p$-norm is not merely a heuristic modification, but a principled way to better align the optimization objective with the contraction geometry of the Bellman operator. This alignment improves error control and remains compatible with gradient-based optimization.

\subsection{Related Works}

Due to the challenges of solving MDPs under function approximation, a large body of work has studied both BRM and DP-based approaches. In this paper, we use the term BRM as a general framework that minimizes the Bellman residual induced by a given Bellman operator. Depending on the choice of operator, BRM can be categorized into (i) evaluation-based BRM, (ii) direct control BRM, and (iii) soft control BRM. Early works introduced the concept of evaluation-based BRM in the context of policy evaluation under linear function approximation, where the residual is defined with respect to the Bellman expectation operator~\cite{baird1995residual}. Subsequent works extended this approach to control settings by incorporating evaluation-based BRM into policy iteration frameworks, where BRM serves as an approximate policy evaluation step, while policy improvement is performed separately via greedy updates~\cite{maillard2010finite, scherrer2010should, antos2008learning, geist2017bellman}. However, in these approaches, the Bellman optimality operator is not directly incorporated into the residual objective. More recent studies have considered direct control BRM formulations that incorporate the Bellman optimality operator into the residual objective~\cite{fujimoto2022should}. Furthermore, recent work has also considered soft control BRM, which replaces the Bellman optimality operator with its entropy-regularized soft variant~\cite{lee2026analysis}. The soft Bellman operator itself has been widely adopted in the literature due to its favorable optimization properties and its close connection to entropy-regularized reinforcement learning~\cite{haarnoja2018soft, haarnoja2017reinforcement, dai2018sbeed}. Our work follows this line by adopting a soft Bellman operator, motivated in particular by the need for gradient-based optimization.

In parallel, DP-based approaches such as projected value iteration (PVI) have been extensively studied. 
While both BRM and PVI can be viewed as approximations of the Bellman equation, it is well known that PVI may lose the contraction property under function approximation, which can lead to instability. Numerous works have attempted to address these issues by analyzing the underlying causes of failure and identifying sufficient conditions for stability~\cite{munos2005error, tsitsiklis1996analysis, de2000existence}, more generally analyzing projected equations and temporal-difference frameworks~\cite{bertsekas2009projected, bertsekas2011temporal}, restricting the function class to preserve contraction properties~\cite{gordon1995stable}, or introducing regularized versions of the projection operator~\cite{lim2024regularized, yang2026periodic}. Motivated by the need to better control error propagation in such non-contractive settings, another line of work has investigated the use of $L_p$-norms in place of the standard $L_2$-norm in DP-based methods~\cite{munos2007performance, farahmand2010error}. Further~\cite{munos2008finite} provided $L_p$-norm regression for PVI and analyzed the error propagation behavior. In contrast, existing BRM approaches predominantly rely on $L_2$-based objectives, and the role of general $L_p$-norm objectives for BRM remains largely unexplored. Our work aims to fill this theoretical gap by studying BRM under general $L_p$-norms and analyzing its alignment with the soft Bellman operator.

\subsection{Contributions}
In this paper, we propose a soft control variant of BRM formulated under a generalized weighted $L_p$-norm. First, we analyze the contraction properties of the soft Bellman operator under the $L_p$-norm. Moreover, we show that although the projection operator can be extended from the $L_2$-norm to the more general $L_p$-norm, the contraction property of the combined Bellman and projection operators no longer holds, and thus the PVI approach is not valid in general. Motivated by this observation, we develop a soft Bellman residual minimization framework under the weighted $L_p$-norm. We further derive performance error bounds and show that, as $p$ increases, the objective progressively aligns with the $L_\infty$-norm Bellman error in the limiting case. Building on this theoretical analysis, we propose a gradient-based optimization algorithm that can stably optimize the resulting objective even for large values of $p$. Finally, we provide empirical evidence that supports our theoretical findings.


\section{PRELIMINARIES}
\subsection{Notation}
Throughout this paper, we adopt standard mathematical notation. We denote the real numbers and the $n$-dimensional Euclidean space by ${\mathbb R}$ and ${\mathbb R}^n$, respectively. The space of all $n \times m$ real matrices is denoted by ${\mathbb R}^{n \times m}$. For a matrix $A$, its transpose is written as $A^\top$, and $I_n$ refers to the identity matrix of dimension $n$. For a finite set $\cal S$, $|{\cal S}|$ denotes its cardinality. $e_i \in {\mathbb R}^n$ denotes the $i$-th standard basis vector. For theoretical analysis, we consider a linear function approximation where $\Phi \in \mathbb{R}^{|\mathcal{S}||\mathcal{A}| \times d}$ denotes a full-rank feature matrix and $\theta \in \mathbb{R}^d$ is the corresponding parameter vector, where $d$ is the number of feature vectors with $d\ll |\mathcal{S}||\mathcal{A}|$. Under this formulation, the parameterized Q-function is represented as $Q_\theta = \Phi\theta$.

\subsection{Markov decision process}
We consider the infinite-horizon discounted MDP, where the agent sequentially takes actions to maximize cumulative discounted rewards. In an MDP with the state-space ${\cal S}:=\{ 1,2,\ldots ,|{\cal S}|\}$ and action-space ${\cal A}:= \{1,2,\ldots,|{\cal A}|\}$, the agent selects an action $a \in {\cal A}$ at the current state $s\in {\cal S}$, then the state
transits to the next state $s'\in {\cal S}$ with transition probability $P(s'|s,a)$, and the transition incurs a
reward $r(s,a,s') \in {\mathbb R}$. We define the expected one-step reward as $R(s,a) := \sum_{s' \in S} P(s'|s,a)\, r(s,a,s')$. A deterministic policy, $\pi :{\cal S} \to {\cal A}$, maps a state $s \in {\cal S}$ to an action $\pi(s)\in {\cal A}$. The objective of an MDP is to find an optimal policy $\pi^*$ that maximizes the expected cumulative discounted reward over an infinite horizon, i.e., $\pi^*:= \arg \max_{\pi\in \Theta} {\mathbb E}\!\left[\sum_{k=0}^\infty \gamma^k R(s_k,a_k)\,\middle|\,\pi\right]$, where $\gamma \in [0,1)$ is the discount factor, $\Theta$ is the set of all deterministic policies, and $(s_0,a_0,s_1,a_1,\ldots)$ is a trajectory generated under $\pi$. Moreover, the Q-function under policy $\pi$ is defined as $Q^{\pi}(s,a)={\mathbb E}\!\left[\sum_{k=0}^\infty \gamma^k R(s_k,a_k)\,\middle|\,s_0=s,a_0=a,\pi\right]$, $(s,a)\in {\cal S} \times {\cal A}$, and the optimal Q-function is $Q^*(s,a)=Q^{\pi^*}(s,a)$. Once $Q^*$ is known, an optimal policy is obtained via $\pi^*(s)=\arg \max_{a\in {\cal A}}Q^*(s,a)$. Throughout, we assume the MDP is ergodic so that the stationary distribution exists.

\subsection{Definitions and lemmas}
In this subsection, we introduce the essential mathematical definitions and lemmas required for our analysis. To address the non-differentiability of the standard Bellman operator in gradient-based optimization, we employ the \textit{soft control Bellman operator}~\cite{haarnoja2017reinforcement, haarnoja2018soft}, denoted by $F_\lambda$. It is defined as
\[
\begin{aligned}
(F_\lambda Q)(s,a)
&:= R(s,a)  + \gamma \sum_{s' \in \mathcal S} P(s'|s,a)\lambda
\ln \Bigg(
\sum_{u \in \mathcal A}
\exp\!\Big(\tfrac{Q(s',u)}{\lambda}\Big)
\Bigg)
\end{aligned}
\]
where $\lambda >0$ is the temperature parameter. It is well known that the soft Bellman operator remains a $\gamma$-contraction in the $L_\infty$-norm, and therefore admits a unique fixed point~\cite{geist2019theory}.

Furthermore, we utilize the \textit{weighted $L_p$-norm}. 
For a vector $x \in \mathbb{R}^n$ and weights $w_i > 0$ where $\sum^{n}_{i=1}w_i=1$, this norm is defined as
\[
\|x\|_{p,w} = \left( \sum_{i=1}^n w_i |x_i|^p \right)^\frac{1}{p}.
\]
For notational convenience, we refer to this norm as the $L_{p,w}$-norm throughout the remainder of the paper. 
To analyze operators under this norm, it is useful to relate $\|\cdot\|_{p,w}$ to more familiar norms such as the standard $L_p$-norm and the $L_\infty$-norm. 
The following lemma, directly adapted from \cite{lee2026unified}, summarizes several well-known inequalities that connect these norms and will be used in our contraction analysis.

\begin{lemma}[{\cite[Lemma 5]{lee2026unified}}]\label{lemma:appendix_lemma}
Let $x \in \mathbb{R}^n$ and $p \ge 1$. Define $w_{\min} := \min_i w_i$ and $w_{\max} := \max_i w_i$. Then the following inequalities hold:
\begin{enumerate}\setlength{\itemsep}{4pt}
    \item $w_{\min}^{1/p}\|x\|_\infty \le \|x\|_{p,w} \le n^{1/p}w_{\max}^{1/p}\|x\|_\infty.$

    \item $w_{\min}^{1/p}\|x\|_p \le \|x\|_{p,w} \le w_{\max}^{1/p}\|x\|_p.$
 
    \item $\|x\|_\infty \le \|x\|_p \le n^{1/p}\|x\|_\infty.$
    
\end{enumerate}
\end{lemma}

\begin{proof}
The proof follows directly from~\cite[Lemma 5]{lee2026unified}.




\end{proof}

\section{Contraction analysis for soft Bellman operator in $L_{p,w}$-norm}
\label{sec:contraction}

In this section, we investigate the contraction property of the soft Bellman operator $F_\lambda$ under the $L_{p,w}$-norm. While the Bellman operator is contractive in the $L_\infty$-norm, this norm is non-differentiable and therefore unsuitable for gradient-based optimization. This motivates the use of the differentiable $L_{p,w}$-norm. We show that, for sufficiently large $p$, $F_\lambda$ remains a contraction in the $L_{p,w}$-norm, which ensures the existence and uniqueness of the fixed point.

\begin{definition}[Effective Contraction Rate]
\label{def:effective_gamma}
For $p \in (1,\infty)$, we define the effective contraction rate of 
$F_\lambda$ in the $L_{p,w}$-norm by
\[
\gamma_{p,w}
:=
\gamma \, n^{\frac{1}{p}}
\left(\frac{w_{\max}}{w_{\min}}\right)^{\frac{1}{p}}.
\]
\end{definition}

\begin{proposition}
\label{prop:lipschitz}
Let $w_i>0$ for all $i \in \{1, 2,\dots,n\}$.
For any $p \in (1,\infty)$,
\[
\|F_\lambda Q - F_\lambda Q'\|_{p,w}
\le
\gamma_{p,w}
\|Q - Q'\|_{p,w}.
\]
\end{proposition}

\begin{proof}
The proof relies on the contraction property of the soft Bellman operator in the $L_\infty$-norm together with the norm equivalence results established in Lemma~\ref{lemma:appendix_lemma}. First recall that the soft Bellman operator satisfies the standard contraction property~\cite{geist2019theory},
\[
\|F_\lambda Q - F_\lambda Q'\|_\infty
\le
\gamma
\|Q-Q'\|_\infty .
\]

\noindent From Lemma~\ref{lemma:appendix_lemma}, we have the norm relations
\[
\|x\|_{p,w}
\le
n^{1/p} w_{\max}^{1/p} \|x\|_\infty ,
\qquad
\|x\|_\infty
\le
w_{\min}^{-1/p}\|x\|_{p,w}.
\]

\noindent Applying the first inequality with $x = F_\lambda Q - F_\lambda Q'$ gives
\[
\|F_\lambda Q - F_\lambda Q'\|_{p,w}
\le
n^{1/p} w_{\max}^{1/p}
\|F_\lambda Q - F_\lambda Q'\|_\infty .
\]

\noindent Using the $L_\infty$ contraction property yields
\[
\|F_\lambda Q - F_\lambda Q'\|_{p,w}
\le
\gamma n^{1/p} w_{\max}^{1/p}
\|Q-Q'\|_\infty .
\]

\noindent Applying the second inequality to $x = Q-Q'$ gives
\[
\|Q-Q'\|_\infty
\le
w_{\min}^{-1/p}
\|Q-Q'\|_{p,w}.
\]

\noindent Substituting this bound into the previous inequality yields
\[
\|F_\lambda Q - F_\lambda Q'\|_{p,w}
\le
\gamma n^\frac{1}{p}
\left(\frac{w_{\max}}{w_{\min}}\right)^\frac{1}{p}
\|Q-Q'\|_{p,w}.
\]

\noindent By Definition~\ref{def:effective_gamma}, this is exactly
\[
\|F_\lambda Q - F_\lambda Q'\|_{p,w}
\le
\gamma_{p,w}
\|Q-Q'\|_{p,w},
\]
which proves the result.
\end{proof}

\begin{corollary}
\label{cor:contraction}
There exists a finite $\bar p \in (1,\infty)$ such that for all $p\ge \bar p$,
\[
\gamma_{p,w} < 1.
\]
Consequently, for sufficiently large $p$, $F_\lambda$ is a contraction in the $L_{p,w}$-norm space
and admits a unique fixed point $Q^*_\lambda$.
\end{corollary}

\begin{proof}
From Proposition~\ref{prop:lipschitz},
\[
\|F_\lambda Q - F_\lambda Q'\|_{p,w}
\le \gamma_{p,w}\|Q-Q'\|_{p,w}, 
\]
\[
\gamma_{p,w} = \gamma n^\frac{1}{p}
\left(\frac{w_{\max}}{w_{\min}}\right)^\frac{1}{p}.
\]
Thus, $F_\lambda$ is a contraction whenever $\gamma_{p,w} < 1$, which is equivalent to
\[
p >
\frac{\ln\!\left(n\frac{w_{\max}}{w_{\min}}\right)}
{\ln(1/\gamma)}.
\]
Hence, there exists a finite $\bar p$ such that $\gamma_{p,w} < 1$ for all $p \ge \bar p$. By the Banach fixed-point theorem~\cite{banach1922operations, bertsekas2012dynamic_1}, it admits a unique fixed point $Q_\lambda^*$.
\end{proof}

Proposition~\ref{prop:lipschitz} and Corollary~\ref{cor:contraction}
demonstrate that the contraction property of the Bellman operator is not confined to the classical $L_\infty$ geometry. 
Since $\gamma_{p,w} \to \gamma$ as $p\to\infty$ and $\gamma<1$, 
the operator becomes contractive in the $L_{p,w}$-norm for sufficiently large $p$.
This establishes that the Bellman fixed-point equation remains well-posed beyond the non-differentiable $L_\infty$ setting, thereby providing a rigorous foundation for the residual minimization framework developed in the next section.

\section{Soft Bellman residual minimization in $L_{p,w}$}
\label{sec:sbrm}

In this section, we consider minimizing the Bellman residual directly in the same geometry in order to approximate the induced fixed-point.

\subsection{Limitation of Projected Value Iteration in $L_{p,w}$-Norm}
\label{subsec:avi_limitation}

To analyze the limitations of PVI, which relies on metric projection, we first define the projection operator. For any $Q \in \mathbb{R}^n$, the $L_{p,w}$-norm metric projection onto the function class
$\mathcal{Q}_\Phi := \{ Q_\theta = \Phi\theta : \theta \in \mathbb{R}^d \}$
is defined as
\begin{equation}
\Gamma_{p,w}(Q)
:=
\arg\min_{Q' \in \mathcal{Q}_\Phi}
\| Q' - Q \|_{p,w}.
\end{equation}

A classical result in approximation theory states that, 
in normed spaces of dimension at least three, 
the property that nearest-point mappings (which are equivalent to metric projections in our setting) 
onto closed convex sets are distance-shrinking characterizes inner product spaces 
\cite{phelps1957convex}. In particular, the following implication holds.

\begin{proposition}[{\cite[Theorem 5.2]{phelps1957convex}}]
\label{prop:nonexpansive}
Let $E$ be a normed space with $\dim E \ge 3$. 
If for every nonempty closed convex set $C \subset E$ 
the metric projection $P_C$ is nonexpansive (whenever it exists), 
then $E$ is a Hilbert space.
\end{proposition}

\noindent From Proposition~\ref{prop:nonexpansive}, and by contraposition, we obtain the following corollary.

\begin{corollary}
\label{cor:projection_failure}
Let $E$ be a Banach space. 
If $E$ is not a Hilbert space, then there exists a nonempty closed convex set 
$C \subset E$ such that the metric projection $P_C$ is not nonexpansive.
\end{corollary}

\noindent Therefore, even if the minimizer defining $\Gamma_{p,w}$ exists and is uniquely attained, 
$\Gamma_{p,w}$ is not necessarily nonexpansive when $p \neq 2$. In other words, the contraction property of metric projections is guaranteed only in inner product spaces (i.e., Hilbert spaces). When $p \neq 2$, the space is no longer Hilbert, and thus the nonexpansiveness of the projection is not guaranteed in general.

Recall that PVI iteratively applies a Bellman backup followed by a projection onto the function class~\cite{gordon1995stable, de2000existence}.
The resulting iteration takes the form
\begin{equation}
\label{eq:pavi_update}
Q_{k+1} = \Gamma_{p,w}\big(F_\lambda Q_k\big).
\end{equation}

\noindent Although Section~\ref{sec:contraction} establishes that $F_\lambda$ is contractive in the $L_{p,w}$-norm for sufficiently large $p$, since $\Gamma_{p,w}$ may fail to be nonexpansive (Corollary~\ref{cor:projection_failure}), the composite operator $\mathcal{T}_{\mathrm{PVI}} := \Gamma_{p,w} \circ F_\lambda$ is not guaranteed to be contractive in $\|\cdot\|_{p,w}$.
Consequently, the contraction-based fixed-point guarantee
for $F_\lambda$ does not directly extend to the projected iteration
(\ref{eq:pavi_update}). An empirical illustration of the loss of contraction is presented in Section~\ref{exp:algo}.
This structural limitation motivates the projection-free
residual minimization approach developed next.

\subsection{Framework formulation of $L_{p,w}$-norm soft Bellman residual minimization}

To circumvent the instability induced by projection, we adopt a projection-free strategy based on soft Bellman residual minimization. 
This approach directly targets the fixed-point equation while remaining within the contraction-preserving geometry of $L_{p,w}$. 
We now formally introduce the \emph{$L_{p,w}$-norm Soft Bellman Residual Minimization} (PSBRM) framework.

\begin{definition}
\label{def:objective}
For $Q_\theta = \Phi\theta$, the soft Bellman residual objective is defined as
\begin{equation}
    f_p(\theta) 
    :=
    \frac{1}{p} 
    \|F_\lambda Q_\theta - Q_\theta\|_{p,w}^p
    =
    \frac{1}{p} 
    \|F_\lambda(\Phi \theta) - \Phi \theta\|_{p,w}^p.
\end{equation}
\end{definition}

\begin{definition}
\label{def:psbrm_problem}
The PSBRM solution is defined as
\begin{equation}
\label{eq:brm_problem}
\theta_p^*
\;:=\;
\arg\min_{\theta \in \mathbb{R}^d}
f_p(\theta).
\end{equation}
\end{definition}

The minimizer $\theta_p^*$ defines the PSBRM estimator,
and the corresponding approximate soft $Q$-function is
$
Q_{\theta_p^*} = \Phi \theta_p^* .
$
The PSBRM framework directly seeks a parameter vector whose induced $Q$-function
approximately satisfies the soft Bellman fixed-point equation,
without introducing an additional projection step.

\subsection{Error analysis and quasi-optimality}
\label{sec:erroranalysis}

We now analyze the approximation properties of the proposed objective. We begin with a fundamental inequality that connects the Bellman residual
to the distance from the fixed point.

\begin{lemma}
\label{lemma:sandwich}
Let $p \ge \bar p$ such that $\gamma_{p,w} < 1 $. For any $\theta \in \mathbb{R}^d$,
\[
\begin{aligned}
\frac{(1-\gamma_{p,w})^p}{p}
\|Q_\theta - Q_\lambda^*\|_{p,w}^p \le
f_p(\theta) \le
\frac{(1+\gamma_{p,w})^p}{p}
\|Q_\theta - Q_\lambda^*\|_{p,w}^p.
\end{aligned}
\]
\end{lemma}

\begin{proof}
Using the triangle inequality, we have
\[
\|F_\lambda Q_\theta - Q_\theta\|_{p,w}
\le
\|F_\lambda Q_\theta - F_\lambda Q_\lambda^*\|_{p,w}
+
\|Q_\lambda^* - Q_\theta\|_{p,w}.
\]
Since $F_\lambda Q_\lambda^* = Q_\lambda^*$ and by Proposition~\ref{prop:lipschitz}, it follows that
\[
\|F_\lambda Q_\theta - F_\lambda Q_\lambda^*\|_{p,w}
\le
\gamma_{p,w}\|Q_\theta - Q_\lambda^*\|_{p,w}.
\]
Then, we obtain
\[
\|F_\lambda Q_\theta - Q_\theta\|_{p,w}
\le
(1+\gamma_{p,w})
\|Q_\theta - Q_\lambda^*\|_{p,w}.
\]
Raising both sides to the $p$-th power and dividing by $p$ gives
\[
f_p(\theta)
\le
\frac{(1+\gamma_{p,w})^p}{p}
\|Q_\theta - Q_\lambda^*\|_{p,w}^p.
\]
For the lower bound, the reverse triangle inequality yields
\[
\|F_\lambda Q_\theta - Q_\theta\|_{p,w}
\ge
\left|
\|Q_\theta - Q_\lambda^*\|_{p,w}
-
\|F_\lambda Q_\theta - F_\lambda Q_\lambda^*\|_{p,w}
\right|.
\]
Using Proposition~\ref{prop:lipschitz} and rearranging the terms yields
\[
\|F_\lambda Q_\theta - Q_\theta\|_{p,w}
\ge
|1-\gamma_{p,w}| \cdot
\|Q_\theta - Q_\lambda^*\|_{p,w}.
\]
Since $\gamma_{p,w}< 1$, we have
\[
\|F_\lambda Q_\theta - Q_\theta\|_{p,w}
\ge
(1-\gamma_{p,w})
\|Q_\theta - Q_\lambda^*\|_{p,w}.
\]
Taking the $p$-th power and dividing by $p$ yields
\[
f_p(\theta)
\ge
\frac{(1-\gamma_{p,w})^p}{p}
\|Q_\theta - Q_\lambda^*\|_{p,w}^p.
\]
Combining the two bounds proves the result.
\end{proof}

Lemma~\ref{lemma:sandwich} shows that the Bellman residual
is tightly controlled by the $L_{p,w}$-norm space distance to the fixed point.
This immediately yields a quasi-optimality guarantee for the global minimizer.

\begin{theorem}
\label{thm:quasi_opt}Let $p \ge \bar p$ such that $\gamma_{p,w} < 1 $.
Let us define
$
\theta_p^* := \arg\min_{\theta \in \mathbb{R}^d} f_p(\theta).
$
Then, the following equation holds:
\[
\|Q_{\theta_p^*} - Q_\lambda^*\|_{p,w}
\;\le\;
\frac{1+\gamma_{p,w}}{1-\gamma_{p,w}}
\min_{\theta \in \mathbb{R}^d}
\|Q_\theta - Q_\lambda^*\|_{p,w}.
\]
\end{theorem}

\begin{proof}
From the inequality in Lemma~\ref{lemma:sandwich}, we have
\[
\begin{aligned}
(1-\gamma_{p,w})
\|Q_\theta - Q_\lambda^*\|_{p,w}
&\le
\|F_\lambda Q_\theta - Q_\theta\|_{p,w} \le
(1+\gamma_{p,w})
\|Q_\theta - Q_\lambda^*\|_{p,w}.
\end{aligned}
\]

Let
\[
\theta_p^* := \arg\min_{\theta \in \mathbb{R}^d} f_p(\theta), 
\qquad
\bar \theta := \arg\min_{\theta \in \mathbb{R}^d}
\|Q_\theta - Q_\lambda^*\|_{p,w}.
\]

Since $\theta_p^*$ minimizes $f_p(\theta)$, we have
\[
\|F_\lambda Q_{\theta_p^*}-Q_{\theta_p^*}\|_{p,w}
\le
\|F_\lambda Q_{\bar \theta}-Q_{\bar \theta}\|_{p,w}.
\]

Applying the bounds from Lemma~\ref{lemma:sandwich} gives
\[
(1-\gamma_{p,w})
\|Q_{\theta_p^*}-Q_\lambda^*\|_{p,w}
\le
(1+\gamma_{p,w})
\|Q_{\bar \theta}-Q_\lambda^*\|_{p,w}.
\]

Therefore
\[
\|Q_{\theta_p^*}-Q_\lambda^*\|_{p,w}
\le
\frac{1+\gamma_{p,w}}{1-\gamma_{p,w}}
\|Q_{\bar \theta}-Q_\lambda^*\|_{p,w}.
\]

Since $\bar \theta$ minimizes the approximation error, we have
\[
\|Q_{\theta_p^*}-Q_\lambda^*\|_{p,w}
\le
\frac{1+\gamma_{p,w}}{1-\gamma_{p,w}}
\min_{\theta \in \mathbb{R}^d}
\|Q_\theta - Q_\lambda^*\|_{p,w}.
\]
\end{proof}

Theorem~\ref{thm:quasi_opt} establishes that the minimizer of the
soft Bellman residual objective is a quasi-best approximation
to the true soft $Q$-function in the $L_{p,w}$-norm.
The multiplicative constant
$
C(p) := \frac{1+\gamma_{p,w}}{1-\gamma_{p,w}}
$
quantifies the approximation gap induced by contraction. This constant $C(p)$ monotonically decreases as $p$ increases, as further discussed in Section~\ref{limiting section}. This behavior justifies that larger values of $p$ yield tighter upper error bounds.

\begin{theorem}
\label{thm:best_compare}
Let $p \ge \bar p$ such that $\gamma_{p,w} < 1$.
Let
\[
\bar{\theta}
:=
\arg\min_{\theta \in \mathbb{R}^d}
\|Q_\theta - Q_\lambda^*\|_{p,w},
\qquad
\theta_p^*
:=
\arg\min_{\theta \in \mathbb{R}^d}
f_p(\theta).
\]
Then,
\[
\|Q_{\bar{\theta}} - Q_{\theta_p^*}\|_{p,w}
\;\le\;
\left(
1 +
\frac{1+\gamma_{p,w}}{1-\gamma_{p,w}}
\right)
\min_{\theta \in \mathbb{R}^d}
\|Q_\theta - Q_\lambda^*\|_{p,w}.
\]
This bound shows that the deviation between the PSBRM solution and the best approximation is bounded by $(1 + C(p))$ times the optimal approximation error, and becomes tighter as $p$ increases.
\end{theorem}

\begin{proof}
By the triangle inequality, we have
\[
\|Q_{\bar{\theta}} - Q_{\theta_p^*}\|_{p,w}
\le
\|Q_{\bar{\theta}} - Q_\lambda^*\|_{p,w}
+
\|Q_{\theta_p^*} - Q_\lambda^*\|_{p,w}.
\]
From Theorem~\ref{thm:quasi_opt}, we have
\[
\|Q_{\theta_p^*} - Q_\lambda^*\|_{p,w}
\le
\frac{1+\gamma_{p,w}}{1-\gamma_{p,w}}
\min_{\theta \in \mathbb{R}^d}
\|Q_\theta - Q_\lambda^*\|_{p,w}.
\]
Substituting this bound yields
\[
\begin{aligned}
\|Q_{\bar{\theta}} - Q_{\theta_p^*}\|_{p,w}
&\le
\|Q_{\bar{\theta}} - Q_\lambda^*\|_{p,w} +
\frac{1+\gamma_{p,w}}{1-\gamma_{p,w}}
\min_{\theta \in \mathbb{R}^d}
\|Q_\theta - Q_\lambda^*\|_{p,w}.
\end{aligned}
\]
Since $\bar{\theta}$ minimizes the approximation error, it follows that
\[
\|Q_{\bar{\theta}} - Q_\lambda^*\|_{p,w}
=
\min_{\theta \in \mathbb{R}^d}
\|Q_\theta - Q_\lambda^*\|_{p,w}.
\]
Consequently, we obtain
\[
\|Q_{\bar{\theta}} - Q_{\theta_p^*}\|_{p,w}
\le
\left(
1+
\frac{1+\gamma_{p,w}}{1-\gamma_{p,w}}
\right)
\min_{\theta \in \mathbb{R}^d}
\|Q_\theta - Q_\lambda^*\|_{p,w}.
\]
\end{proof}

Taken together, Theorems~\ref{thm:quasi_opt}
and~\ref{thm:best_compare}
show that the PSBRM solution lies within a
contraction-controlled constant factor of the
best approximation error achievable within
the function class.
In particular, PSBRM does not introduce
an uncontrolled additional approximation bias;
any deviation from the best approximation
is explicitly bounded by the contraction-dependent
constant $C(p)$.

\section{Limiting Analysis of PSBRM}
\label{limiting section}

\subsection{Monotonic tightening of the quasi-optimality coefficient}

In this section, we characterize how increasing the exponent $p$ affects both
(i) the quasi-optimality bound from Section~\ref{sec:sbrm} and
(ii) the limiting geometry of the residual objective.
Throughout, recall the effective contraction rate
$\gamma_{p,w}$ and the corresponding coefficient $C(p)$.

\begin{theorem}
\label{thm:mono_Cp}
Let us define 
\[
C(p):=\frac{1+\gamma_{p,w}}{1-\gamma_{p,w}}.
\]
Then, there exists a finite threshold $\bar p$ such that for all $p_1>p_2\ge \bar p$,
\[
C(p_1)<C(p_2).
\]
Moreover,
\[
\lim_{p\to\infty} C(p)
=
\frac{1+\gamma}{1-\gamma}.
\]
\end{theorem}

\begin{proof}
Recall that
\[
\gamma_{p,w}
=
\gamma \, n^{\frac{1}{p}}
\left(\frac{w_{\max}}{w_{\min}}\right)^{\frac{1}{p}}.
\]
Let $A := n \frac{w_{\max}}{w_{\min}}$, which satisfies $A > 1$ for $n > 1$.
Then \(A^{1/p}\) is strictly decreasing in \(p\), and hence \(\gamma_{p,w}\) is also strictly decreasing in \(p\). On the contraction regime \(p\ge \bar p\), we have \(\gamma_{p,w}\in[0,1)\). Define
\[
\phi(x):=\frac{1+x}{1-x}, \qquad x\in[0,1).
\]
Since
\[
\phi'(x)=\frac{2}{(1-x)^2}>0,
\]
the function \(\phi\) is strictly increasing on \([0,1)\). Therefore,
\[
C(p)=\frac{1+\gamma_{p,w}}{1-\gamma_{p,w}}=\phi(\gamma_{p,w})
\]
is strictly decreasing in \(p\). This proves the first claim. Finally, since \(A^{1/p}\to 1\) as \(p\to\infty\), we have \(\gamma_{p,w}\to \gamma\). Hence,
\[
\lim_{p\to\infty} C(p)=\frac{1+\gamma}{1-\gamma}.
\]
\end{proof}

Theorem~\ref{thm:mono_Cp} shows that, within the contraction regime,
larger $p$ systematically reduces the multiplicative gap in the quasi-optimality guarantee
(Theorem~\ref{thm:quasi_opt}).

\begin{figure}[h]
    \centering
    \includegraphics[width=0.5\textwidth]{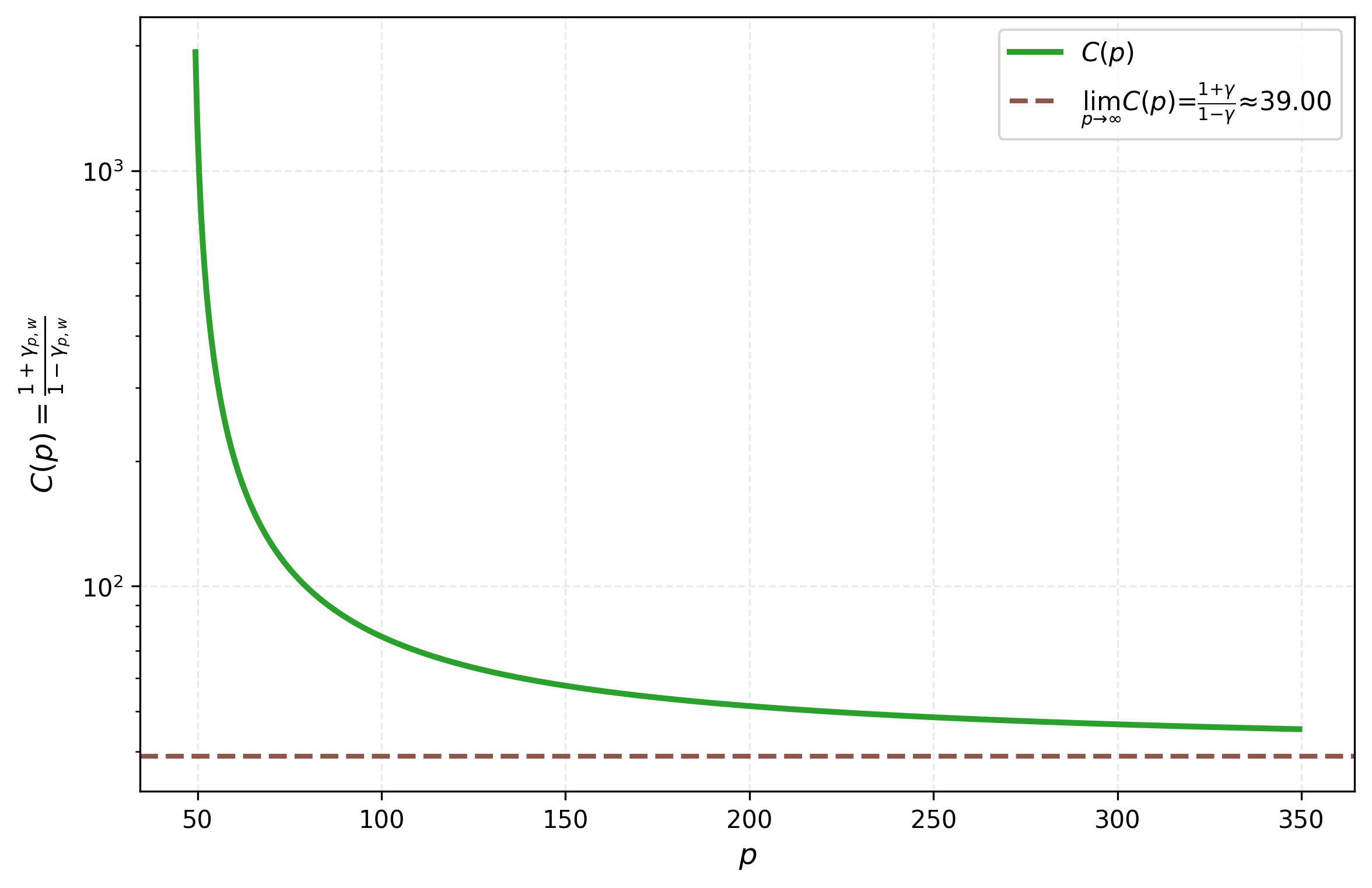}
    \caption{Monotonic decay of $C(p)$ and its convergence as $p \to \infty$.}
    \label{fig:curve}
\end{figure}

Figure~\ref{fig:curve} illustrates the asymptotic behavior of $C(p)$ for the specific case $\gamma = 0.95$.
The constant $C(p)$ decreases monotonically and converges to its limiting value as $p \to \infty$. 
This behavior demonstrates that the multiplicative gap in the quasi-optimality guarantee becomes progressively tighter as $p$ increases.


\subsection{Asymptotic alignment with the $L_{\infty}$ geometry}

While Theorem~\ref{thm:mono_Cp} quantifies improvement at the level of the bound,
we next formalize how the residual objective itself approaches the classical $L_\infty$ contraction geometry. Let $\delta_\theta := F_\lambda Q_\theta - Q_\theta \in \mathbb{R}^n$ denote the Bellman residual.
Let us define the normalized residual functional
\[
J_p(\theta)
:=
\bigl(p f_p(\theta)\bigr)^{1/p}
=
\|\delta_\theta\|_{p,w},
\qquad
J_\infty(\theta)
:=
\|\delta_\theta\|_{\infty}.
\]

\begin{proposition}
\label{prop:uniform_limit}
For any $\theta\in\mathbb{R}^d$, we have
\[
\lim_{p\to\infty} J_p(\theta) = J_\infty(\theta).
\]
Moreover, the convergence is uniform on any compact subset of $\mathbb{R}^d$.
\end{proposition}

\begin{proof} By Lemma~\ref{lemma:appendix_lemma}, for any \(x\in\mathbb{R}^n\),
\[
w_{\min}^{1/p}\|x\|_\infty
\le
\|x\|_{p,w}
\le
n^{1/p}w_{\max}^{1/p}\|x\|_\infty.
\]
Applying this with \(x=\delta_\theta\), we obtain
\[
w_{\min}^{1/p}J_\infty(\theta)
\le
J_p(\theta)
\le
n^{1/p}w_{\max}^{1/p}J_\infty(\theta).
\]
Since
\[
w_{\min}^{1/p}\to 1,
\qquad
n^{1/p}w_{\max}^{1/p}\to 1
\qquad\text{as }p\to\infty,
\]
we have
\[
\lim_{p\to\infty}J_p(\theta)=J_\infty(\theta)
\]
for every \(\theta\in\mathbb{R}^d\).

Now let \(K\subset\mathbb{R}^d\) be compact. Since \(\delta_\theta\) depends continuously on \(\theta\), the map
\[
\theta\mapsto J_\infty(\theta)=\|\delta_\theta\|_\infty
\]
is continuous on \(K\). Hence there exists a constant \(M_K<\infty\) such that
\[
J_\infty(\theta)\le M_K
\qquad
\text{for all }\theta\in K.
\]
From the above norm inequalities, we have
\[
J_p(\theta)-J_\infty(\theta)
\le
\bigl(n^{1/p}w_{\max}^{1/p}-1\bigr)J_\infty(\theta),
\]
and
\[
J_\infty(\theta)-J_p(\theta)
\le
\bigl(1-w_{\min}^{1/p}\bigr)J_\infty(\theta).
\]
Therefore, for all \(\theta\in K\),
\[
|J_p(\theta)-J_\infty(\theta)|
\le
\max\!\Bigl\{
n^{1/p}w_{\max}^{1/p}-1,\,
1-w_{\min}^{1/p}
\Bigr\}
\,M_K.
\]
The factor on the right converges to \(0\) as \(p\to\infty\). Hence
\[
\sup_{\theta\in K}|J_p(\theta)-J_\infty(\theta)|\to 0,
\]
which proves uniform convergence on every compact subset of \(\mathbb{R}^d\).
\end{proof}

Proposition~\ref{prop:uniform_limit} states that increasing $p$ does not merely tighten a bound but
continuously deforms the optimization geometry from an average-type residual measure
toward the worst-case residual measure in $L_{\infty}$.


\section{Algorithms for PSBRM}

We now describe a practical optimization procedure for computing the estimator. The analysis in the previous sections shows that minimizing the Bellman residual in $L_{p,w}$-norm yields a quasi-optimal approximation to the soft Bellman fixed point when the contraction condition is satisfied under sufficiently large $p$. This motivates solving the PSBRM problem using gradient-based optimization. In general, the objective $f_p(\theta)$ involves absolute values and is therefore nonsmooth. To simplify the derivation of the gradient and the resulting optimization algorithm, we restrict the exponent $p$ to be an even integer with $p \ge 2$. Under this assumption, the objective becomes a polynomial function of the Bellman residual and is differentiable everywhere, allowing us to apply standard gradient-based methods~\cite{boyd2004convex}. Based on this observation, we derive the gradient of the PSBRM objective and construct a gradient descent algorithm to compute the parameter vector $\theta$.

\begin{lemma}
\label{lem:gradient}
Let us define 
\[
f_p(\theta)
:=
\frac{1}{p}\|F_\lambda(\Phi\theta)-\Phi\theta\|_{p,w}^p,
\qquad
\delta_\theta := F_\lambda(\Phi\theta)-\Phi\theta.
\]
Suppose that \(p\ge 2\) is an even integer. Then \(f_p\) is differentiable and
\[
\nabla_\theta f_p(\theta)
=
\sum_{i=1}^n
w_i
\bigl(e_i^\top \delta_\theta\bigr)^{p-1}
\nabla_\theta\!\bigl(e_i^\top \delta_\theta\bigr).
\]
Moreover, if \(Q_\theta=\Phi\theta\) and \(\pi_\theta\) denotes the Boltzmann (soft greedy) policy induced by \(Q_\theta\), namely
\[
\pi_\theta(j\mid i)
:=
\frac{\exp(Q_\theta(i,j)/\lambda)}
{\sum_{u\in\mathcal A}\exp(Q_\theta(i,u)/\lambda)},
\]
then the Jacobian of the residual satisfies
\[
\nabla_\theta \delta_\theta
=
(\gamma P\Pi^{\pi_\theta}-I_n)\Phi,
\]
where \(\Pi^{\pi_\theta}\) is the linear policy operator defined by
\[
(\Pi^{\pi_\theta} Q)(i)
=
\sum_{a\in\mathcal A}\pi_\theta(a\mid i)\,Q(i,a).
\]
Consequently,
\[
\nabla_\theta\!\bigl(e_i^\top \delta_\theta\bigr)
=
\Phi^\top(\gamma P\Pi^{\pi_\theta}-I_n)^\top e_i.
\]
\end{lemma}

\begin{proof}
By definition,
\[
f_p(\theta)
=
\frac{1}{p}\sum_{i=1}^n w_i |e_i^\top\delta_\theta|^p,
\qquad
\delta_\theta := F_\lambda(\Phi\theta)-\Phi\theta.
\]
Since \(p\) is even, \(|x|^p = x^p\), and hence \(f_p\) is differentiable. 
Applying the chain rule yields
\[
\begin{aligned}
\nabla_\theta f_p(\theta)
&=
\frac{1}{p}\sum_{i=1}^n
w_i \nabla_\theta (e_i^\top\delta_\theta)^p \\
&=
\sum_{i=1}^n
w_i
(e_i^\top\delta_\theta)^{p-1}
\nabla_\theta(e_i^\top\delta_\theta).
\end{aligned}
\]
We now compute \(\nabla_\theta(e_i^\top\delta_\theta)\).
Recall the soft Bellman operator,
\[
\begin{aligned}
(F_\lambda Q)(s,a)
&= R(s,a) \\
&\quad + \gamma \sum_{s'} P(s'\mid s,a)\,
\lambda \ln \sum_{u\in\mathcal A}
\exp\!\Bigl(\tfrac{Q(s',u)}{\lambda}\Bigr),
\end{aligned}
\]
and defining
\[
V_Q(s)
:=
\lambda \ln \sum_{u\in\mathcal A}
\exp\!\bigl(Q(s,u)/\lambda\bigr),
\]
we obtain
\[
\nabla_Q V_Q = \Pi^{\pi_Q},
\qquad
\nabla_Q F_\lambda(Q) = \gamma P \Pi^{\pi_Q}.
\]
Evaluating at \(Q=Q_\theta=\Phi\theta\) gives
\[
\nabla_Q F_\lambda(Q_\theta)
=
\gamma P \Pi^{\pi_\theta}.
\]
By the chain rule, we have
\[
\nabla_\theta F_\lambda(\Phi\theta)
=
(\gamma P\Pi^{\pi_\theta})\Phi,
\]
and therefore
\[
\nabla_\theta\delta_\theta
=
(\gamma P\Pi^{\pi_\theta}-I_n)\Phi.
\]
Hence, it follows that
\[
\nabla_\theta(e_i^\top\delta_\theta)
=
\Phi^\top(\gamma P\Pi^{\pi_\theta}-I_n)^\top e_i.
\]
Substituting this into the expression for \(\nabla_\theta f_p(\theta)\)
completes the proof.
\end{proof}

\begin{algorithm}[h]
\caption{Normalized Gradient Descent for PSBRM}
\label{alg:psbrm_ngd}
\begin{algorithmic}[1]
\State Initialize $\theta_0 \in \mathbb{R}^d$ and choose an even integer $p \ge 2$

\For{\(k=0,1,2,\dots\)}
    \State Compute the residual
    \[
    \delta_{\theta_k}=F_\lambda(\Phi\theta_k)-\Phi\theta_k
    \]
    \State Normalize the residual
    \[
    \tilde{\delta}_{\theta_k}
    =
    \frac{\delta_{\theta_k}}{\|\delta_{\theta_k}\|_\infty}
    \]
    \State Compute the normalized gradient
    \[
    \tilde{\nabla}_\theta f_p(\theta_k)
    =
    \sum_{i=1}^n
    w_i
    (e_i^\top \tilde{\delta}_{\theta_k})^{p-1}
    \nabla_\theta(e_i^\top \delta_{\theta_k})
    \]
    \State Update
    \[
    \theta_{k+1}
    =
    \theta_k-\alpha_k \tilde{\nabla}_\theta f_p(\theta_k)
    \]
\EndFor
\end{algorithmic}
\end{algorithm}

Once the gradient is established, it can be directly used within a deterministic gradient descent algorithm to update the parameter vector $\theta$. However, when \(p\) is large, the gradient term \((e_i^\top \delta_\theta)^{p-1}\) can become highly sensitive to the scale of the Bellman residual. In particular, if some residual components have magnitude greater than one, the resulting gradient may become excessively large, which makes the optimization numerically unstable and the choice of learning rate $\alpha_k$ highly sensitive. To improve stability, we therefore normalize the residual by its maximum absolute component before forming the update direction.

Let us define
\[
\tilde{\delta}_\theta
:=
\frac{\delta_\theta}{\|\delta_\theta\|_\infty}.
\]
We then use the normalized update direction
\[
\tilde{\nabla}_\theta f_p(\theta)
=
\sum_{i=1}^n
w_i
(e_i^\top \tilde{\delta}_\theta)^{p-1}
\nabla_\theta(e_i^\top \delta_\theta).
\]
Since
\[
(e_i^\top \tilde{\delta}_\theta)^{p-1}
=
\frac{(e_i^\top \delta_\theta)^{p-1}}{\|\delta_\theta\|_\infty^{p-1}},
\]
this normalization rescales the original gradient by the positive scalar \(\|\delta_\theta\|_\infty^{-(p-1)}\). Hence,
\[
\tilde{\nabla}_\theta f_p(\theta)
=
c(\theta)\,\nabla_\theta f_p(\theta),
\qquad c(\theta)>0.
\]
Therefore, \(\tilde{\nabla}_\theta f_p(\theta)=0\) if and only if \(\nabla_\theta f_p(\theta)=0\), 
and the normalization does not alter the set of stationary points of the objective. Based on this observation, we implement a normalized version of the gradient descent algorithm as Algorithm~\ref{alg:psbrm_ngd} 
in which the Bellman residual is scaled by its maximum magnitude at each iteration.

\section{EXPERIMENTS}

\subsection{Experiment setup}

We consider a simple MDP with six states and two actions, with a discount factor of $\gamma = 0.95$.
The transition probabilities for the two actions, denoted by $P_0$ and $P_1$, are specified as follows, where 
$P(s' \mid s, a=0)=P_0$ and $P(s' \mid s, a=1)=P_1$. 
We also use the following reward matrix $R$. 
In addition, we use a randomly generated full-rank feature matrix $\Phi \in \mathbb{R}^{12 \times 6}$.

\[
P_0=
\begin{bmatrix}
0 & 1 & 0 & 0 & 0 & 0\\
0 & 0 & 1 & 0 & 0 & 0\\
1 & 0 & 0 & 0 & 0 & 0\\
0 & 0 & 0 & 0 & 1 & 0\\
0 & 0 & 0 & 0 & 0 & 1\\
0 & 0 & 0 & 1 & 0 & 0
\end{bmatrix}, \;
P_1=
\begin{bmatrix}
0 & 0 & 0 & 1 & 0 & 0\\
0 & 0 & 0 & 0 & 1 & 0\\
0 & 0 & 0 & 0 & 0 & 1\\
1 & 0 & 0 & 0 & 0 & 0\\
0 & 1 & 0 & 0 & 0 & 0\\
0 & 0 & 1 & 0 & 0 & 0
\end{bmatrix}, \;
R=
\begin{bmatrix}
0.8 & 1.2\\
0.8 & -0.4\\
1.0 & 0.2\\
0.2 & 0.6\\
-0.6 & 0.4\\
-0.8 & 0.3
\end{bmatrix}.
\]

\subsection{Algorithm comparison}
\label{exp:algo}
In this section, to examine the performance of PSBRM and compare it with other baseline methods, we consider three algorithms. 
The first is projected value iteration under the $L_{p}$-norm, denoted by $L_p$-PVI, which was introduced earlier in Section~\ref{subsec:avi_limitation} and is included in this section to illustrate algorithmic instability. 
The other two are our proposed method, PSBRM, and its basic variant based on soft Bellman residual minimization under the $L_2$-norm, which we denote by $L_2$-SBRM. 
The implementation of PSBRM follows Algorithm~\ref{alg:psbrm_ngd} with a sufficiently large exponent $p=80$, whereas $L_2$-SBRM is obtained from the same update rule by setting $p=2$.

As shown in Figure~\ref{fig:algo_comparison}, the $L_p$-PVI algorithm exhibits clear instability. In contrast, the Bellman residual minimization family shows stable error trajectories. 
Moreover, the overall error of $L_2$-SBRM is larger than that of our proposed method, PSBRM. 
This result is well aligned with our claims established in Theorem~\ref{thm:quasi_opt}, Theorem~\ref{thm:best_compare}, and Theorem~\ref{thm:mono_Cp}, which suggest that larger values of $p$ are better aligned with the geometry of the $L_\infty$-norm and yield tighter upper bounds on the error.

We further compare PSBRM with the $L_2$-PVI algorithm. As shown in Figure~\ref{fig:algo_avi}, $L_2$-PVI exhibits clear divergence, whereas PSBRM maintains stable convergence. This behavior is consistent with the well-known result that PVI can diverge due to norm mismatch~\cite{gordon1995stable, de2000existence}. In contrast, PSBRM satisfies the contractive property in our analysis, which ensures stable convergence.

\begin{figure}[h]
    \centering
    \begin{subfigure}[t]{0.32\textwidth}
        \centering
        \includegraphics[width=\textwidth]{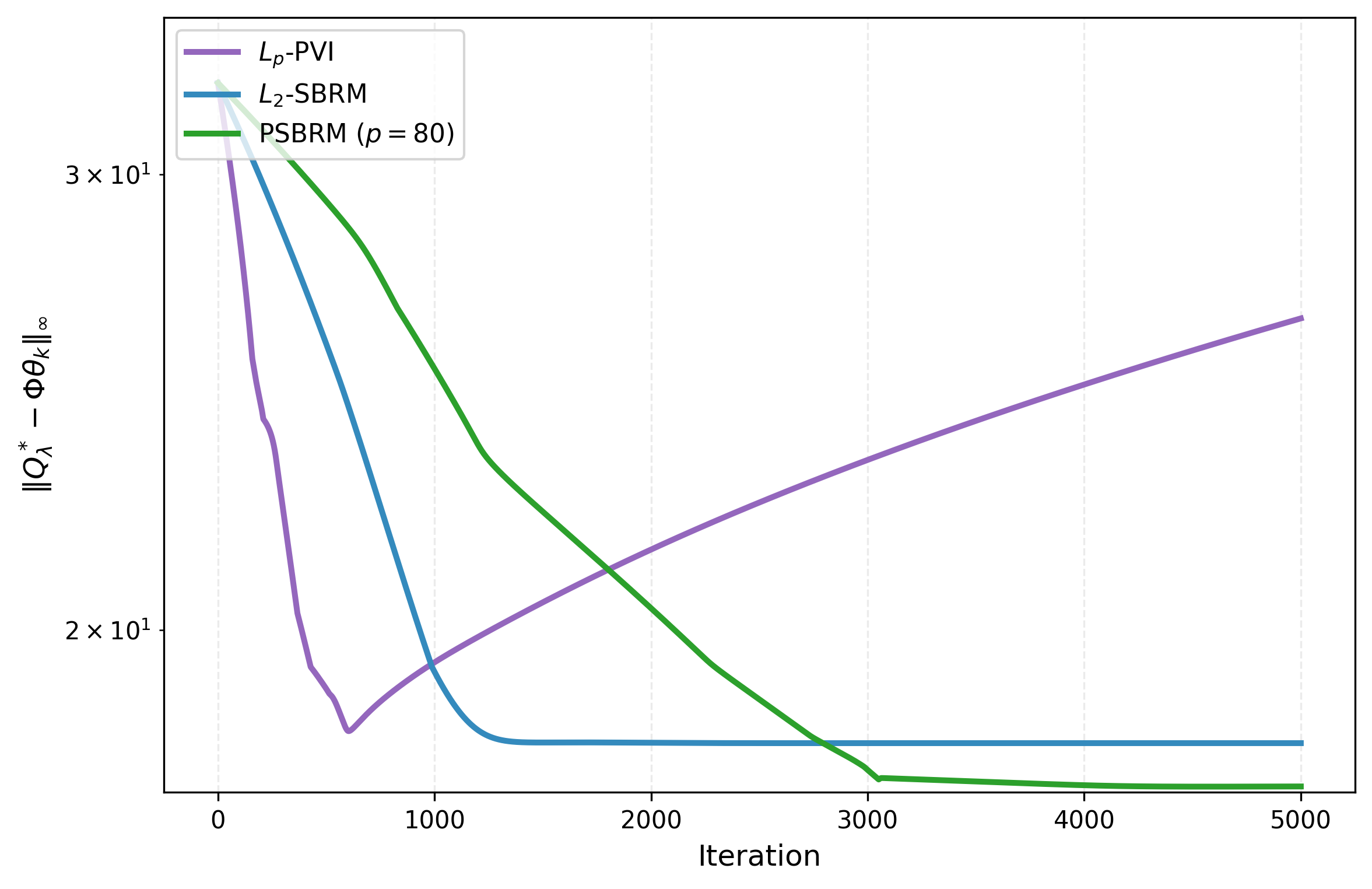}
        \caption{Comparison of methods. We compare $L_p$-PVI, $L_2$-SBRM, and PSBRM $(p=80)$ on the simple MDP.}
        \label{fig:algo_comparison}
    \end{subfigure}
    \hfill
    \begin{subfigure}[t]{0.32\textwidth}
        \centering
        \includegraphics[width=\textwidth]{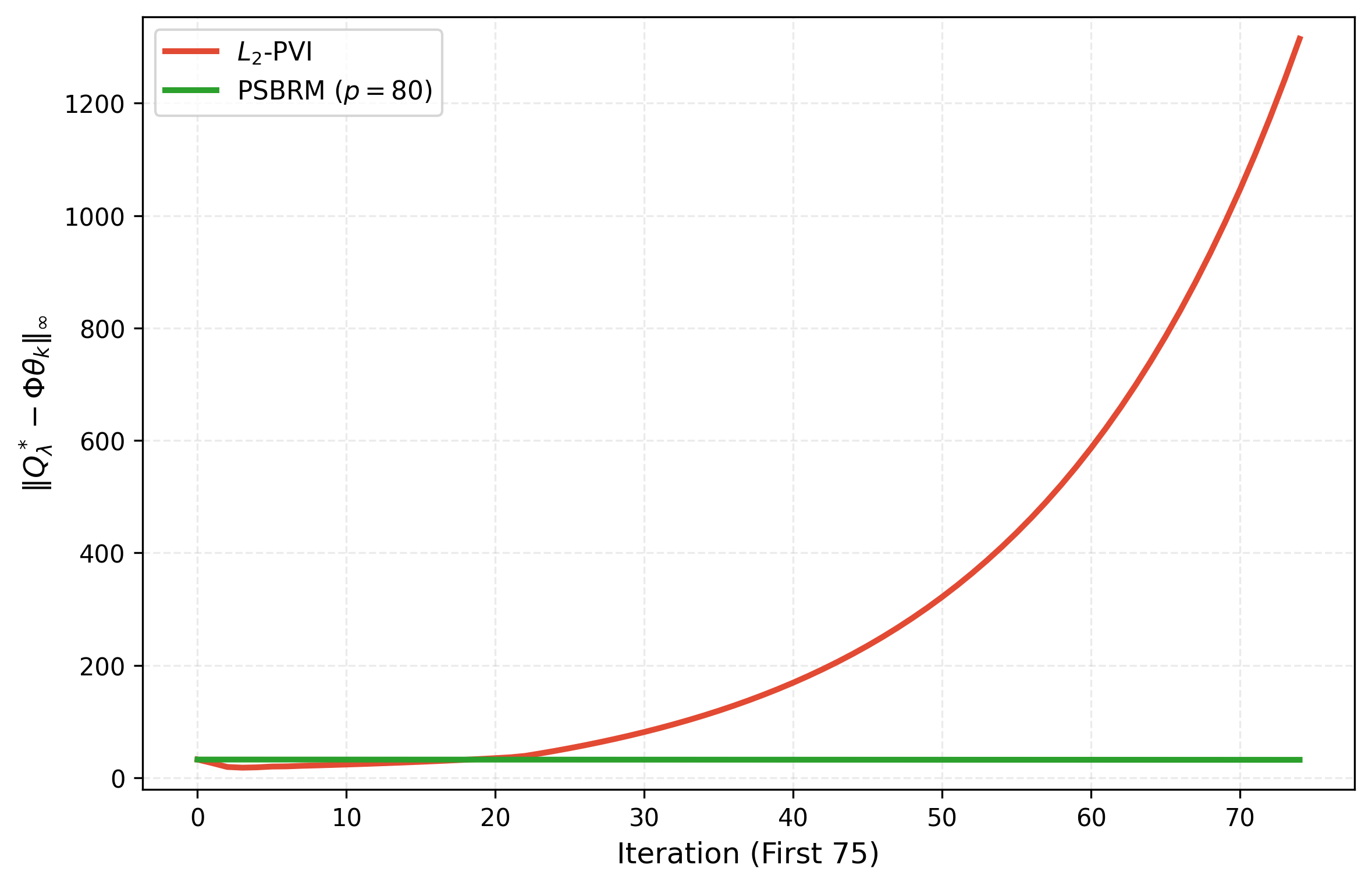}
        \caption{Comparison Between $L_2$-PVI, and PSBRM $(p=80)$.}
        \label{fig:algo_avi}
    \end{subfigure}
    \hfill
    \begin{subfigure}[t]{0.32\textwidth}
        \centering
        \includegraphics[width=\textwidth]{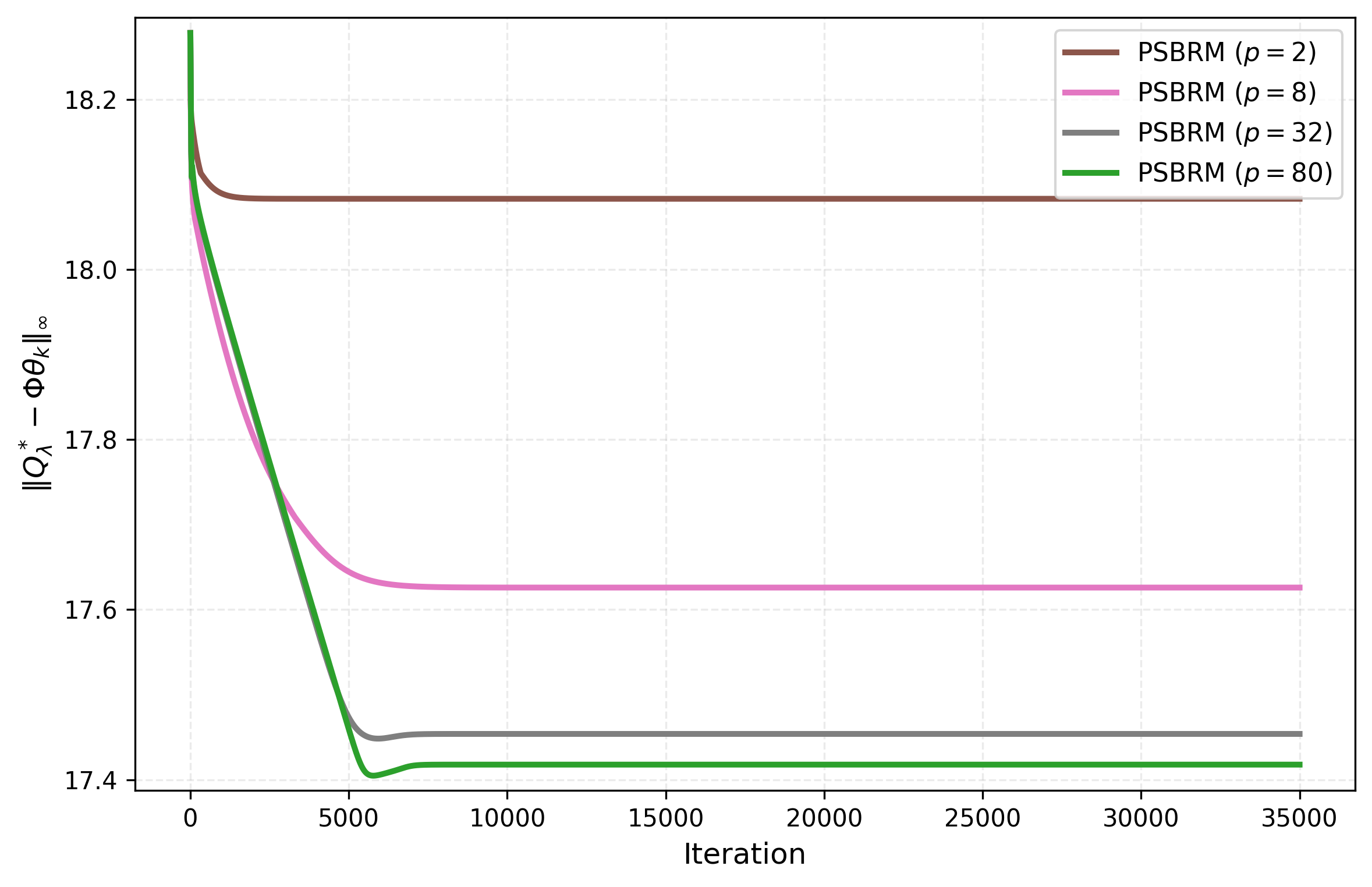}
        \caption{Ablation study of PSBRM with different values of $p \in \{2, 8, 32, 80\}$ on the simple MDP.}
        \label{fig:ablation}
    \end{subfigure}
    \caption{Experimental results on the simple MDP.}
    \label{fig:combined_results}
\end{figure}

\subsection{Effect of $p$ on PSBRM}

In this section, we conduct the same experimental setting as in Section~\ref{exp:algo}, but restrict attention to PSBRM with different values of $p$. 
The values of $p$ are chosen as $[2, 8, 32, 80]$ so that we can examine the effect of increasing $p$.

As shown in Figure~\ref{fig:ablation}, we observe that as $p$ increases, the error at the stationary point decreases monotonically. 
This result is also well aligned with our claims established in Theorem~\ref{thm:quasi_opt}, Theorem~\ref{thm:best_compare}, and Theorem~\ref{thm:mono_Cp}. 
Although the theory only proves the monotonic decrease of the upper bound on the error, the experiment shows that the actual error also decreases monotonically.





\bibliographystyle{IEEEtran}
\bibliography{reference}



\section*{ACKNOWLEDGMENT}

The work was supported by the Institute of Information Communications Technology Planning Evaluation (IITP) funded by the Korea government under Grant 2022-0-00469, and the BK21 FOUR
from the Ministry of Education (Republic of Korea).

\end{document}